\documentclass[sigconf]{aamas}  

\AtBeginDocument{%
  \providecommand\BibTeX{{%
    \normalfont B\kern-0.5em{\scshape i\kern-0.25em b}\kern-0.8em\TeX}}}
    
\usepackage{algorithm}
\usepackage{algorithmic}
\usepackage{enumerate}
\usepackage{enumitem}
\usepackage{framed}
\usepackage{subfigure}
\usepackage{booktabs}    
\usepackage{flushend} 

\setcopyright{ifaamas}  
\copyrightyear{2020} 
\acmYear{2020} 
\acmDOI{} 
\acmPrice{} 
\acmISBN{} 
\acmConference[AAMAS'20]{Proc.\@ of the 19th International Conference on Autonomous Agents and Multiagent Systems (AAMAS 2020)}{May 9--13, 2020}{Auckland, New Zealand}{B.~An, N.~Yorke-Smith, A.~El~Fallah~Seghrouchni, G.~Sukthankar (eds.)}  

\begin{document}

\title{Redistribution Mechanism on Networks}  


\author{Wen Zhang}
\affiliation{%
  \institution{ShanghaiTech University}
  \city{Shanghai} 
  \country{China}
}
\email{zhangwen@shanghaitech.edu.cn}
\author{Dengji Zhao}
\affiliation{%
  \institution{ShanghaiTech University}
  \city{Shanghai} 
  \country{China}
}
\email{zhaodj@shanghaitech.edu.cn}
\author{Hanyu Chen}
\affiliation{%
	\institution{ShanghaiTech University}
	\city{Shanghai} 
	\country{China}
}
\email{chenhy1@shanghaitech.edu.cn}

\begin{abstract}
  Redistribution mechanisms have been proposed for more efficient resource allocation but not for profit. We consider redistribution mechanism design in a setting where participants are connected and the resource owner is only connected to some of them. In this setting, to make the resource allocation more efficient, the resource owner has to inform the others who are not her neighbours, but her neighbours do not want more participants to compete with them. Hence, the goal is to design a redistribution mechanism such that participants are incentivized to invite more participants and the resource owner does not earn or lose much money from the allocation. We first show that existing redistribution mechanisms cannot be directly applied in the network setting and prove the impossibility to achieve efficiency without a deficit. Then we propose a novel network-based redistribution mechanism such that all participants on the network are invited, the allocation is more efficient and the resource owner has no deficit.
\end{abstract}

%


\maketitle


\section{Introduction}

The problem of resource allocation is about to decide the allocation of some resource among a group of self-interested agents. Since the valuation for the resource differs from various agents, it is a natural objective for the resource owner to pursue the efficiency of the allocation, i.e., allocating the resource to the agent with the highest valuation. In many scenarios, the owner does not really aim at making profits but hopes the wealth maintained among the agents. For example, the government wants to build a library in a community that values it the most; a charity distributes a donation to the recipient who needs it the most; a hospital allocates doctors to rural areas where doctors are highly demanded.

To find the agent with the highest valuation, one common solution is to hold an auction such as the well-known Vickrey-Clarke-Groves (VCG) mechanism~\cite{krishna2009auction,vickrey1961counterspeculation,clarke1971multipart,groves1973incentives}. However, the payments under VCG will be delivered to the auctioneer, which againsts our non-profit purpose. To maintain as much wealth as possible among the participants, many redistribution mechanisms based on VCG have been proposed~\cite{Cavallo:2006:ODM:1160633.1160790,DBLP:conf/aaai/Guo11}. They redistribute the revenue generated by VCG back to all participants. However, these mechanisms can only be applied in static settings, i.e., the resource owner can only allocate the item to the person whom she can directly contact with (her neighbours). 

Therefore, in our setting, we have another issue: how can the owner enroll more participants in the resource allocation problem in order to achieve a more efficient allocation? Advertising is a widely used method to disseminate information to attract more people. However, it should be paid in advance without a guarantee that there will be more participants or a more efficient allocation. Moreover, it is irrational for a resource owner who no longer cares about profit to pay something for the allocation. Therefore, in this paper, we consider a cost-free promotion by incentivizing participants to invite their neighbours via their social connections~\cite{easley2010networks}, which is an enormous challenge as no one would be willing to invite more competitors without a profit.

Hence, in this paper, we propose a new network-based redistribution mechanism to tackle this challenge, where the reward redistributed to each agent is a monotone increasing function to the number of participants she invites. Although the agents are not paid in advance, they are still incentivized not only to report their valuation truthfully but also to diffuse the information to all their neighbours without sacrificing the non-deficit guarantee. This is one of the key features of our mechanism. Eventually, more agents will be informed about the resource allocation and a more efficient allocation will be achieved. Moreover, it also satisfies the desirable properties of traditional redistribution mechanisms such as individual rationality and asymptotically budget-balance.

Some interesting work related to information diffusion via networks have been studied recently. \citeauthor{DBLP:conf/aaai/LiHZZ17}~\shortcite{DBLP:conf/aaai/LiHZZ17} proposed an auction mechanism where the seller sells one item in a social network and the participants are inviting each other to attract more participants. With this inspiration, soon afterwards \citeauthor{Zhao:2018:SMI:3237383.3237400}~\shortcite{Zhao:2018:SMI:3237383.3237400} generalized the mechanism for multiple homogeneous items in the same setting. Their attention is on how to maximize the seller's revenue, which is different from ours. We aim for a more efficient allocation without profits by attracting more participants. We refer to the idea of their work and design our redistribution mechanism to achieve the goal.

There exists a rich literature on redistribution mechanisms for resource allocation~\cite{DBLP:conf/ijcai/Guo19,moulin2009almost,gujar2011redistribution,guo2012worst}. Furthermore, redistribution mechanisms have also been extended to the settings of public projects~\cite{guo2011budget,naroditskiy2012redistribution,guo2013undominated,DBLP:conf/ijcai/Guo19}. However, none of the work took the natural connections between participants into account. More to the point, our mechanism promises desirable properties when participants are connected via their personal connections which cannot be achieved by the existing mechanisms.

We claim our main contributions here. First, to the best of our knowledge, we are the very first to study the redistribution mechanism design problem in social networks. Second, we show the limitations of the classical Cavallo mechanism if it is directly extended in social networks and prove the impossibility to achieve efficiency without a deficit in our setting. Third, we propose a novel network-based redistribution mechanism which improves the efficiency of the allocation without sacrificing all the desirable properties.

The structure of the paper is organized as follows. Section~\ref{section:pre} describes the background and basic definitions of the problem. Section~\ref{section:Cavallo} extends the Cavallo mechanism in social networks and discusses its limitations. After that, we propose our network-based redistribution mechanism for tree structures, show its outstanding properties and prove the efficiency impossibility result in Section~\ref{section:Mech_Tree}. Finally, we generalize our mechanism in graphs in Section~\ref{section:Mech_Graph} and discuss future work.

\section{Preliminaries}
\label{section:pre}
We consider a setting where an owner $o$ wants to allocate an item in a social network $G=(V,E)$, where each agent $i\in V$ is a potential bidder with a private valuation $v_i\geq0$ for the item. Each agent $i\in V\cup \{o\}$ has a private neighbour set $r_i\subseteq V$. If there exists an edge $e(i,j)\in E$ from agent $i$ to agent $j$, we say $j$ is $i$'s neighbour, denoted by $j\in r_i$. Let $d_i$ be the depth of agent $i\in V$, which is the length of the shortest path from the owner to $i$. We say agent $j$ is agent $i$'s child neighbour if $j\in r_i$ and $d_j=d_i+1$, denoted by $j\in r_i^c$. The objective of the owner is to allocate the item to the agent with the highest valuation to the best of her ability and maintain as much wealth as possible among the agents. That is, she is aiming to minimize the surplus of the payment transfers in the mechanism. 

Initially, without any third-party platforms, the owner can only allocate the item to her neighbours since all the other agents cannot be reached directly. To attract more potential bidders, a feasible approach is to ask the agents to invite their neighbours to join the allocation. However, there is no reason for these bidders to invite more competitors. Thus, how to design the incentives for the agents to propagate the information without sacrificing desirable properties is the greatest challenge we need to overcome.

In this paper, we propose a novel network-based redistribution mechanism, where all the agents are willing not only to report their private valuation for the item but also to invite all their neighbours to the mechanism voluntarily.

We start by defining some notations in the mechanism:

\begin{itemize}[leftmargin= 20 pt]
	\item Let $\theta_i=(v_i,r_i)$ be the type of agent $i\in V$, which is $i$'s true private information. 
	\item Let $\theta=(\theta_1,\dots,\theta_n)=(\theta_{-i},\theta_i)$ be the type profile of all the agents, where $\theta_{-i}$ is the type profile for agents except $i$.
	\item Let $\Theta_i$ be the type space for agent $i\in V$ and $\Theta=(\Theta_1,\dots,\Theta_n)=(\Theta_{-i},\Theta_i)$ be the type profile space for all the agents.
	\item Let $\hat{\theta}_i=(\hat{v}_i,\hat{r}_i)$ be the reported type of agent $i\in V$, where $\hat{v}_i$ is the valuation she reported and $\hat{r}_i$ is the neighbour set she has invited. Let $\hat{\theta}_i=nil$, if agent $i$ is not invited.
	\item Let $G(\hat{\theta})=(V(\hat{\theta}),E(\hat{\theta}))$ be the graph generated by the reported type profile $\hat{\theta}$, which is constructed in the following way:
	$$
	\begin{cases}
	i\in V(\hat{\theta})& \text{if $i=o$ or $i\in \hat{r}_j$ where $j\in V(\hat{\theta})$}\\
	e(i,j)\in E(\hat{\theta})& \text{if $e(i,j)\in E$ and $i,j\in V(\hat{\theta})$}
	\end{cases}$$
\end{itemize}

Note that the reported type is not definitely the same as the true type. Therefore, we can easily observe that for each agent $i\in V$,

$$
\begin{cases}
\hat{\theta}_{i}\not=nil& \text{if $i\in V(\hat{\theta})$}\\
\hat{\theta}_{i}=nil& \text{if $i\not\in V(\hat{\theta})$}
\end{cases}$$

\begin{figure}[htbp]
	\centering
	\subfigure[]{%
		\label{chart:network}%
		\includegraphics[width=0.5\linewidth]{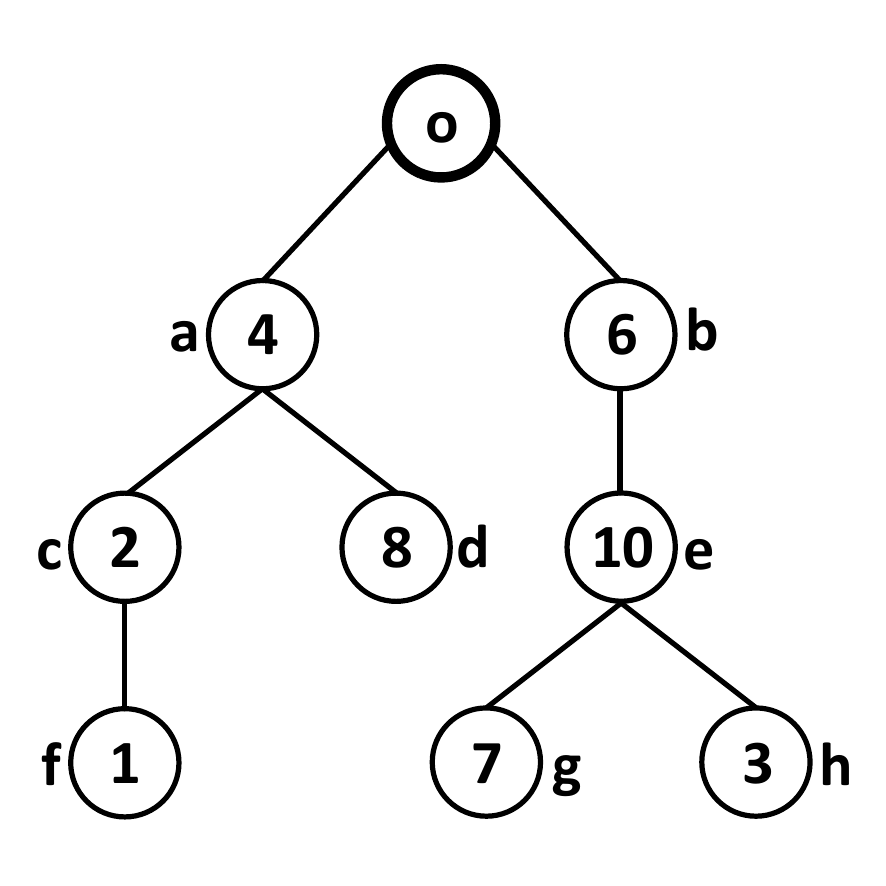}}%
	\subfigure[]{%
		\label{chart:graph}%
		\includegraphics[width=0.5\linewidth]{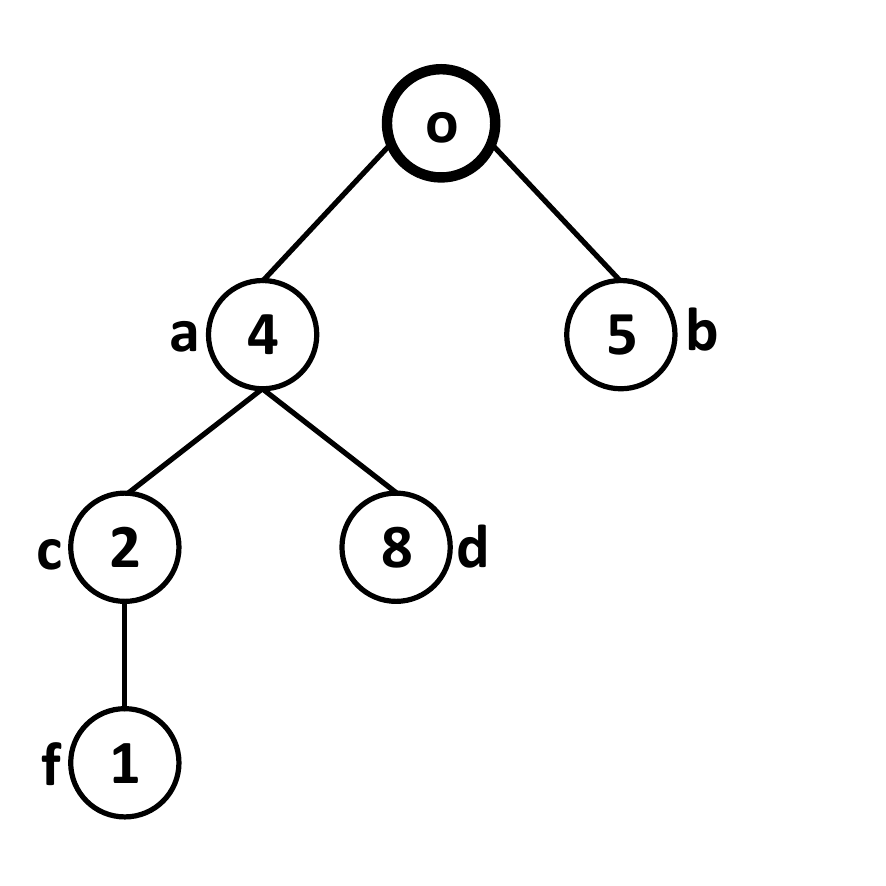}}\\
	\caption{An example of a social network $G$ and its generated graph $G(\hat{\theta})$.}
	\label{feasible}
\end{figure}

Figure~\ref{feasible} shows an example of the graph generation process. Figure~\ref{chart:network} is a real social network, where $o$ is the resource owner and other nodes are the bidders. The letter beside each node represents the ID of a bidder and the value in the node represents the bidder's private valuation for the item. Suppose that $\hat{\theta}$ is the reported type profile where all the bidders report the truthful type except that bidder $b$ misreports $\hat{\theta}_b=(5,\emptyset)$, which means that she misreports her valuation to be $5$ and does not invite her neighbour $e$. Then the corresponding graph $G(\hat{\theta})$ will be generated as Figure~\ref{chart:graph} shows, where bidder $e$, $g$ and $h$ do not occur as they cannot receive the information without $b$'s invitation and thus their reported type should all be $nil$.

Intuitively, if an agent is not informed about the allocation, she will not occur in the generated graph. Thus, there is no need to take such agents into consideration. In the following discussion, we only focus on the generated graph.

\begin{definition}
	A redistribution mechanism  $\mathcal{M}$ in the social network is defined by an allocation policy $\pi=(\pi_1, \pi_2,\dots, \pi_n)$ and a payment policy $p=(p_1,p_2,\dots, p_n)$, where $\pi_i:\Theta\rightarrow \{0,1\}$ and $p_i:\Theta\rightarrow\textbf{R}$.
\end{definition}

Given a reported type profile $\hat{\theta}\in \Theta$, the payment policy $p(\hat{\theta})=(p_1(\hat{\theta}),\dots,p_n(\hat{\theta}))$ represents the money paid by each agent and the allocation policy $\pi(\hat{\theta})=(\pi_1(\hat{\theta}),\dots,\pi_n(\hat{\theta}))$ represents the item allocation result. Intuitively, we have

$$
\begin{cases}
p_{i}(\hat{\theta})\geq0& \text{if agent $i$ pays $p_i$ to the owner}\\
p_{i}(\hat{\theta})<0& \text{if agent $i$ receives $-p_i$ from the owner}
\end{cases}$$

$$\pi_{i}(\hat{\theta})=
\begin{cases}
1& \text{if the item is allocated to agent $i$}\\
0& \text{if the item is not allocated to agent $i$}
\end{cases}$$

Since there is only one item to be allocated, it is natural to require that the allocation policy will not allocate the resource to those who have not participated and no more than one agent will win the item.

\begin{definition}
	We say an allocation $\pi(\hat{\theta})$ is feasible if at most one agent $i$ with $\hat{\theta}_i\not=nil$ is allocated the item, i.e., $$\sum_{i\in V, \hat{\theta}_i\not=nil}\pi_i(\hat{\theta})\leq1$$
\end{definition}

In our setting, we assume that there is no cost for an agent to spread the information to her neighbours. Therefore, given a reported type profile $\hat{\theta}$ of all the agents, the utility of agent $i\in V$ with type $\theta_i$ and reported type $\hat{\theta}_i$ is defined as:
\begin{center}
	$u_i(\theta_i,\hat{\theta})=\pi_i(\hat{\theta})v_i-p_i(\hat{\theta})$
\end{center}

Therefore, the surplus of the payment transfer of the mechanism is defined as:

\begin{center}
	$S(\hat{\theta})=\sum_{i\in V}p_i(\hat{\theta})$
\end{center}

\begin{definition}
	Given a reported type profile $\hat{\theta}\in\Theta$ and a feasible allocation $\pi$, the social welfare of allocation $\pi(\hat{\theta})$ is defined by $$SW(\hat{\theta})=\sum_{i\in V}\pi_{i}(\hat{\theta})\hat{v}_i$$
\end{definition}

That is, the social welfare of an allocation is the sum of the reported valuations of all agents who win the item for this allocation. The higher the social welfare is, the more \textbf{efficient} the allocation is. 

\begin{definition}
	A redistribution mechanism $\mathcal{M}=(\pi,p)$ is \textbf{individually rational} (IR) if for all $i\in V$ and all $\hat{\theta}\in\Theta$, we have $u_i(\theta_i,\hat{\theta})\geq0$, where $\hat{\theta}_i=(v_i, \hat{r}_i)$.
\end{definition}

This is a general extension of the traditional definition of individual rationality. That is, all the agents participated in the mechanism will not have negative utilities as long as she truthfully reports her private valuation. Note that the definition does not require the agents to invite all their neighbours, which loosens the restriction of reporting true type.

\begin{definition}
	A redistribution mechanism $\mathcal{M}=(\pi,p)$ is \textbf{incentive compatible} (IC) if for all $i\in V$ and all $\hat{\theta},\hat{\theta}'\in \Theta$, we have $u_i(\theta_i,\hat{\theta})\geq u_i(\theta_i,\hat{\theta}')$, where $\hat{\theta}_i=\theta_i$ and $\hat{\theta}_i'\not=\theta_{i}$. $\hat{\theta}'$ is the corresponding reported type profile when $i$ changes her reported type such that the reported type of any agent $j\not\in V(\hat{\theta}')$ is $nil$ and the others are the same as those in $\hat{\theta}$.
\end{definition}

For the traditional definition of incentive compatibility, all the buyers' dominant strategy is to truthfully report their private valuation of the item. Here, we put forward a stricter extended definition of IC for the network setting, where all the agents are incentivized not only to report valuation truthfully but also to invite all their neighbours.

\begin{definition}
	A redistribution mechanism $\mathcal{M}=(\pi,p)$ is \textbf{non-deficit} (ND) if for all $i\in V$ and all $\hat{\theta}\in \Theta$, we have $$S(\hat{\theta})=\sum_{i\in V}p_i(\hat{\theta})\geq0$$
\end{definition}

That is, the surplus of the payment transfer is non-negative, which is reasonable because the owner or other outside parties has to pay for the deficit otherwise.

\begin{definition}
	A redistribution mechanism $\mathcal{M}=(\pi,p)$ is \textbf{asymptotically budget-balanced} (ABB) if for all $\hat{\theta}\in \Theta$, we have $$\lim\limits_{|N| \to +\infty} S(\hat{\theta})=0$$
\end{definition}

This is to say when the number of the participants goes to infinity, almost all the money received by the owner will be redistributed back to the participants.

\section{Cavallo Mechanism in Social Networks}
\label{section:Cavallo}
Considering the constraint of generalized IR, extended IC, ND and ABB, seemingly some traditional redistribution mechanisms can be easily applied to the new setting in social networks. Therefore, in this section, we first review the classical Cavallo mechanism~\cite{Cavallo:2006:ODM:1160633.1160790} and show that it may lead to a deficit and disincentivize agents to diffuse the information.

The Cavallo mechanism modifies the VCG framework and redistributes the transfer payments back among the agents while keeping the specified desirable properties of VCG. The mechanism for a single item is outlined below:

\begin{framed}
	\noindent\textbf{Cavallo Mechanism}
	
	\noindent\rule{\textwidth}{0.5pt}
	
	\begin{enumerate}
		\item Each agent $i\in V$ submits her reported type $\hat{\theta}_i=(\hat{v}_i,\hat{r}_i)$, which forms a type profile $\hat{\theta}\in \Theta$ and a generated graph $G(\hat{\theta})$. 
		\item The mechanism chooses the highest bidder $w\in\arg\max_{i\in V}\hat{v}_i$ as the winner, which maximizes the social welfare, and allocates the item to her.
		\item All the agents' payments are defined by 	$p_i(\hat{\theta})=p_i^{VCG}(\hat{\theta})-p_i^{re}(\hat{\theta})$, where	$p_i^{VCG}(\hat{\theta})=SW(\hat{\theta}_{-i})-(SW(\hat{\theta})-\pi_i(\hat{\theta})\hat{v}_i)$ is the money paid for auction and $p_i^{re}(\hat{\theta})=\frac{S_{i}^{VCG}}{n}$ is the money redistributed to $i$.
		\item The surplus of the mechanism which is given to the resource owner is $S(\hat{\theta})=\sum_{i\in V}p_i(\hat{\theta})$.
	\end{enumerate}
\end{framed}

Intuitively, the Cavallo mechanism can be viewed as two stages: the auction stage and the redistribution stage. In the first auction stage, the item will be allocated to the highest bidder and she pays the loss of other players because of her participation to the owner as defined in the traditional VCG mechanism. Then in the second redistribution stage, the owner redistributes the money received to all the agents in the mechanism. The money redistributed to agent $i$ is calculated by $\frac{S_{i}^{VCG}}{n}$, where $S_{i}^{VCG}$ is the surplus lower-bound in VCG among the same agents over all possible reported valuation of $i$. Specially, in the single-item setting, the payment in the first stage of the highest bidder is the second highest reported valuation. Let $m_j$ be the $j^{th}$ highest bidder among all the agents $V$. The redistributed money is $\frac{\hat{v}_{m_3}}{n}$ for the highest bidder and the second highest bidder, and $\frac{\hat{v}_{m_2}}{n}$ for the others. Consequently, all the agents share the surplus and the rewards are independent of their reported valuation.

Although the Cavallo mechanism is IR and ABB, we then show that it may run a deficit and agents may be not willing to diffuse the information in social networks. 

\begin{figure}[htbp]
	\centering
	\subfigure[]{%
		\label{chart:deficit}%
		\includegraphics[width=0.45\columnwidth]{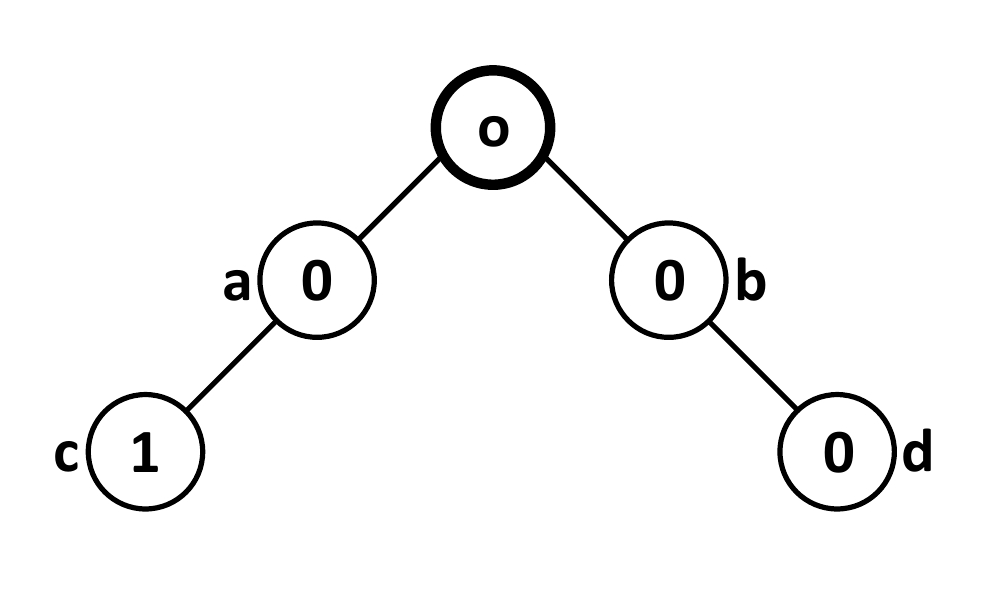}}%
	\subfigure[]{%
		\label{chart:disincentivize}%
		\includegraphics[width=0.45\columnwidth]{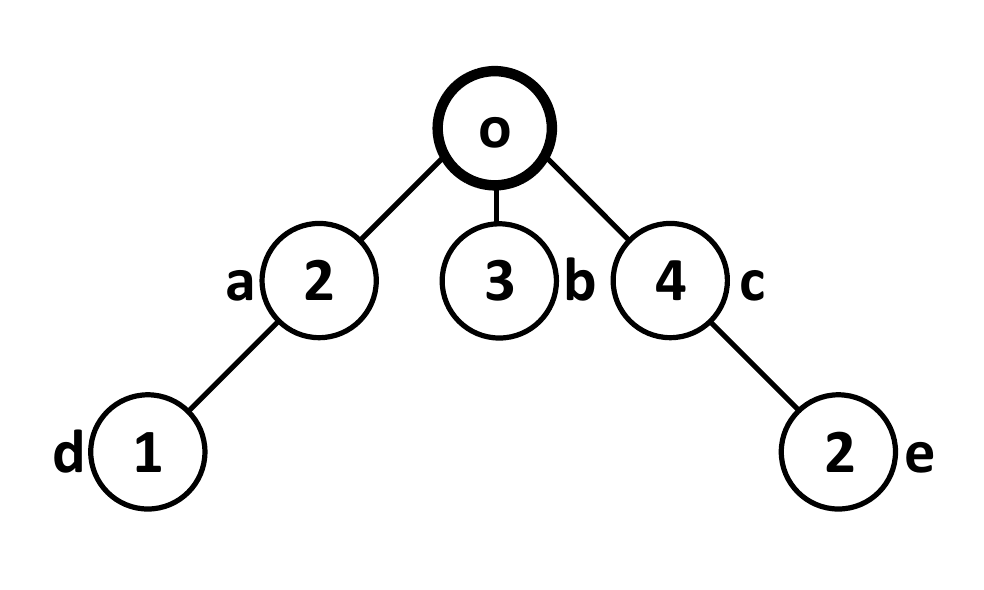}}\\
	\caption{(a) An example of running a deficit; (b) An example of disincentivizing the diffusion.}
	\label{proposition}
\end{figure}

\begin{proposition}\label{pro:deficit}
	The Cavallo mechanism runs a deficit.
\end{proposition}
\begin{proof}
	We prove the proposition by showing a counter example in Figure~\ref{chart:deficit}. Since agent $c$ is the only one with positive valuation, the mechanism will allocate the item to her. For agent $b$ and $d$, their participation does not affect the result, thus they will pay nothing in the first auction stage. For agent $c$, if she does not participate in the mechanism, no one else will win the item with positive utility, thus her payment is also zero. For agent $a$, if she quits the mechanism, agent $c$ will not be involved in, so her payment is $-1$. Since there is only one agent with positive valuation, all the agents will be redistributed nothing in the redistribution stage. Thus we have $S(\hat{\theta})=\sum_{i\in V}p_i(\hat{\theta})=-1$, which runs a deficit. 
\end{proof}

\begin{proposition}\label{pro:incentive}
	The Cavallo mechanism disincentivizes the agents to diffuse the information.
\end{proposition}
\begin{proof}
	By showing a counter example in Figure~\ref{chart:disincentivize}, we can easily prove the proposition. As agent $c$ is the highest bidder, she keeps the item and pays the second highest valuation $p_c^{VCG}(\hat{\theta})=\hat{v}_b=3$. Then in the redistribution stage, all these $5$ agents will share the surplus. For agent $b$ and $c$, $S_b^{VCG}=S_c^{VCG}=\hat{v}_a=2$, then $p_b^{re}(\hat{\theta})=p_c^{re}(\hat{\theta})=2/5$. For agent $a$, $d$ and $e$, $S_a^{VCG}=S_d^{VCG}=S_e^{VCG}=\hat{v}_b=3$, then $p_a^{re}(\hat{\theta})=p_d^{re}(\hat{\theta})=p_e^{re}(\hat{\theta})=3/5$. However, if $a$ stops inviting $d$ and $c$ stops inviting $e$, the allocation and the surplus will remain the same but the number of agents who share the surplus will decrease. Then $p_a^{re}(\hat{\theta}')=3/3=1>p_a^{re}(\hat{\theta})$ and $p_c^{re}(\hat{\theta}')=2/3>p_c^{re}(\hat{\theta})$. Thus, the Cavallo mechanism disincentivizes the agents' diffusion.
\end{proof}

Proposition~\ref{pro:deficit} and~\ref{pro:incentive} show that owing to the special constraint of social networks, extending the Cavallo mechanism into network settings simply is not feasible. In the following section, we will introduce our novel mechanism with all the desirable properties satisfied.

\section{Redistribution Mechanism in Trees}
\label{section:Mech_Tree}
To tackle the challenges on networks, we propose a network-based redistribution mechanism (NRM) which satisfies all the desirable properties mentioned before. In this section, we will first start with a special type network, tree structures, which provides a clearer presentation of the intuition behind. Later we will generalize our mechanism on common graphs. 

\begin{definition}
	Given a reported type profile $\hat{\theta}\in\Theta$ of all the agents and the generated tree graph $G=(V,E)$, for each agent $i,j\in V$ if there exists a simple path from the seller $s$ to $j$ through $i$ with the depth relation $d_i<d_j$, we say $i$ is $j$'s \textbf{ancestor} and $j$ is $i$'s \textbf{descendant}.
\end{definition}

Some basic notations in the mechanism is defined as:

\begin{itemize}[leftmargin= 20 pt]
	\item Let $A_i=(a_1,\cdots,a_k)$ be the \textbf{ancestor sequence} of agent $i$, where $a_j\in A_i$ is an ancestor of agent $i$ and $d_{a_1}<d_{a_2}<\cdots<d_{a_k}$.
	\item Let $G_i=(V_i,E_i)$ be the \textbf{subtree} of agent $i$ if $G_i$ is a tree consisting of $i$ and all its descendants in $G$. Let $n_i=|V_i|$ be the number of agents in $G_i$.
	\item Let $B_{a_j}=r_{a_{j-1}}^c\setminus {a_j}$ be the \textbf{sibling set} of agent $a_j\in A_i$, where all the agents in $B_{a_j}$ has the same parent as $a_j$.
	\item Let $\hat{v}_D^{(1)}$ denote the highest reported valuation among all the agents in any set $D$.
\end{itemize}

That is, for each agent $i\in V$, she cannot join in the mechanism if any agent in her ancestor sequence $A_i$ does not diffuse the information. Besides, without the invitation of agent $i$, any agent in her subtree $V_i$ cannot receive the information.

Now we will propose our NRM in Trees. The detailed procedure is given in Algorithm~\ref{NRM-T}.

\begin{algorithm}[t]
	\caption{Network-based Redistribution Mechanism}
	\begin{algorithmic}[1]
		\REQUIRE ~~\\ 
		A type profile $\hat{\theta}\in\Theta$ and the item owner $o$;
		\STATE construct the tree $G=(V,E)$ by $\hat{\theta}$;
		\STATE identify the highest bidder $h\in\arg\max_{i\in V}\hat{v}_i$;
		\STATE find the ancestor sequence $A_h$;
		\STATE find the sibling set $B_i$ for each agent $i\in A_h\cup\{h\}$;
		\STATE set $A=(o, A_h, h)=(a_0,a_1,\cdots,a_k,a_{k+1})$;
		\STATE initialize $\pi_i=0$, $R_i=0$ and $p_i=0$ for each agent $i$;
		\STATE initialize $p_{a_0}^{auc}=0$ and $S(\hat{\theta})=0$;
		\FOR{each $a_j$ in $(a_1,\cdots,a_{k+1})$}
		\STATE $p_{a_j}^{auc}=\hat{v}_{V\setminus V_{a_j}}^{(1)}$;
		\STATE $S_{a_j}=p_{a_j}^{auc}-p_{a_{j-1}}^{auc}$;
		\STATE $X=B_{a_j}\cup a_j$;
		\STATE $n_X= \sum_{q\in X} n_q$;
		\FOR{each $k\in X$}
		\STATE $h'\in\arg\max_{i\in V\setminus V_k}\hat{v}_i$;
		\STATE find $A_{h'}$ and $A'=(o,A_{h'},h')$;
		\STATE $S_{-k}=
		\begin{cases}
		\hat{v}_{V\setminus \{V_{a'_j}\cup V_k\}}^{(1)}-p_{a_{j-1}}^{auc} & \text{if $a_{j-1}=a'_{j-1}\in A'$} \\
		0 & \text{otherwise}
		\end{cases}
		$;
		\STATE $R_k=\frac{n_k}{n_X}\cdot S_{-k}$;
		\STATE $p_k=-R_k$;
		\ENDFOR
		\STATE Update surplus $S(\hat{\theta})= S(\hat{\theta})+S_{a_j}-\sum_{k\in B_{a_j}\cup a_j}R_k$;
		\IF{$\hat{v}_{a_j}\geq\hat{v}_{V\setminus V_{a_j}}^{(1)}$}
		\STATE $\pi_{a_j}=1$;
		\STATE $p_{a_j}=\hat{v}_{V\setminus V_{a_j}}^{(1)}-R_{a_j}$;
		\STATE break;
		\ENDIF
		\ENDFOR
		\STATE Return $\pi_i$ and $p_i$ for each agent $i$ and $S(\hat{\theta})$ for the owner;
	\end{algorithmic}
	\label{NRM-T}
\end{algorithm}

Intuitively, although the network-based redistribution mechanism is centralized, it can be viewed as a sequential procedure. The item passes through the ancestor sequence of the highest bidder $h$ and each agent $a_j\in A\cup\{h\}$ is required to pay the highest reported valuation without her participation for either passing or keeping the item. The item is allocated to the first agent whose valuation is higher than or equal to her required payment. The money each agent paid first compensates the last ancestor's payment and the remaining part will be redistributed among her siblings and herself. The money redistributed to agent $i$ is the new required payment difference multiplied by the percentage of agents in $i$'s subtree over all the agents in the subtrees of the ancestor and its siblings considered, which is a monotone increasing function to the number of their descendants. The more their descendants are, the more they will be redistributed, which incentivizes the agents' diffusion. The rest money which is not redistributed will be given to the owner as the surplus. Note that NRM is a centralized mechanism and all the operation process is run by the owner. Therefore, only the winner in NRM is the one who is required to pay the money to the owner while the ancestor sequence of the winner and their siblings will be redistributed rewards.

\begin{figure}[t]
	\centering
	\subfigure[]{%
		\label{chart:NRM_1}%
		\includegraphics[width=0.5\linewidth]{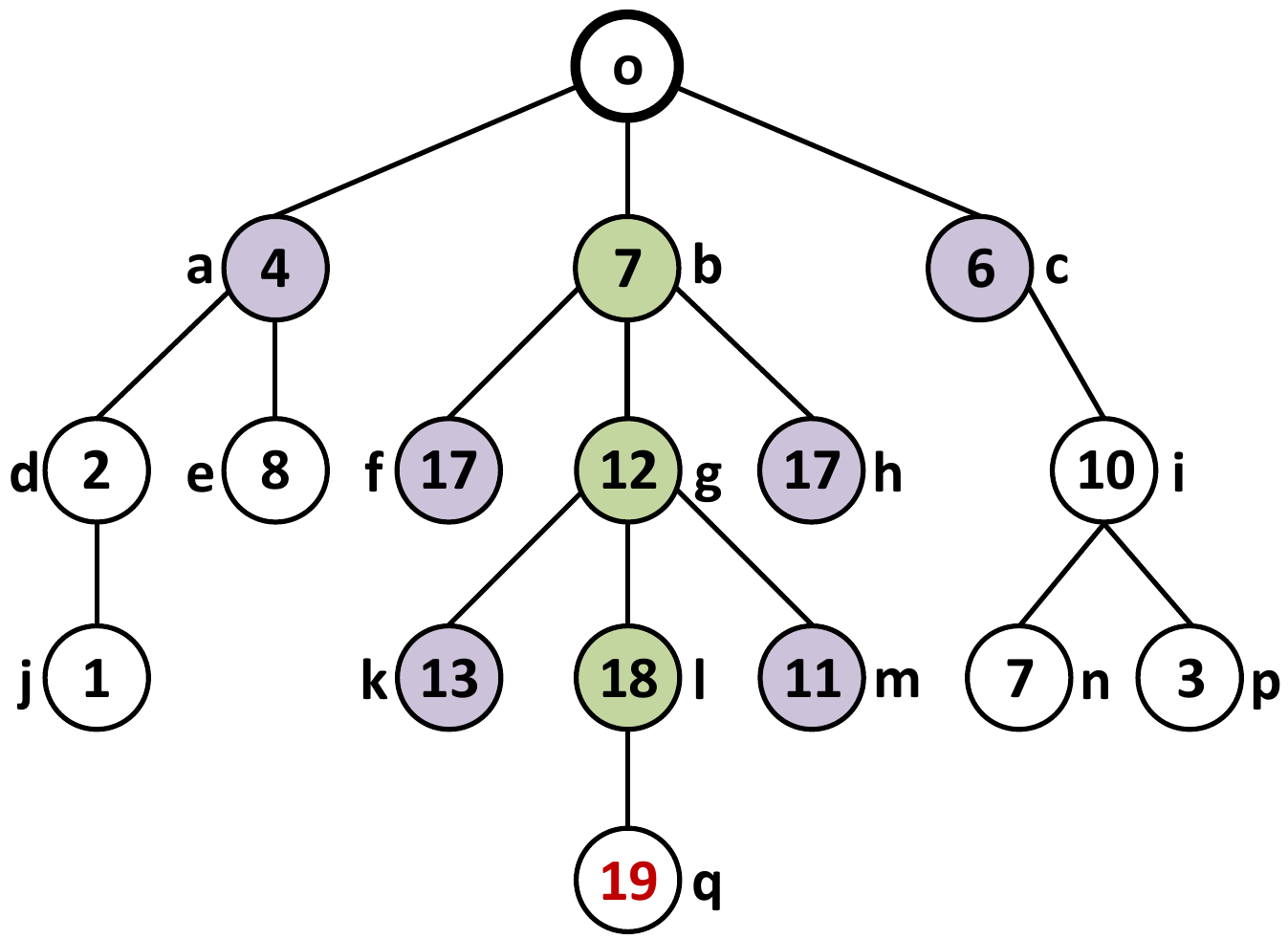}}%
	\subfigure[]{%
		\label{chart:NRM_2}%
		\includegraphics[width=0.5\linewidth]{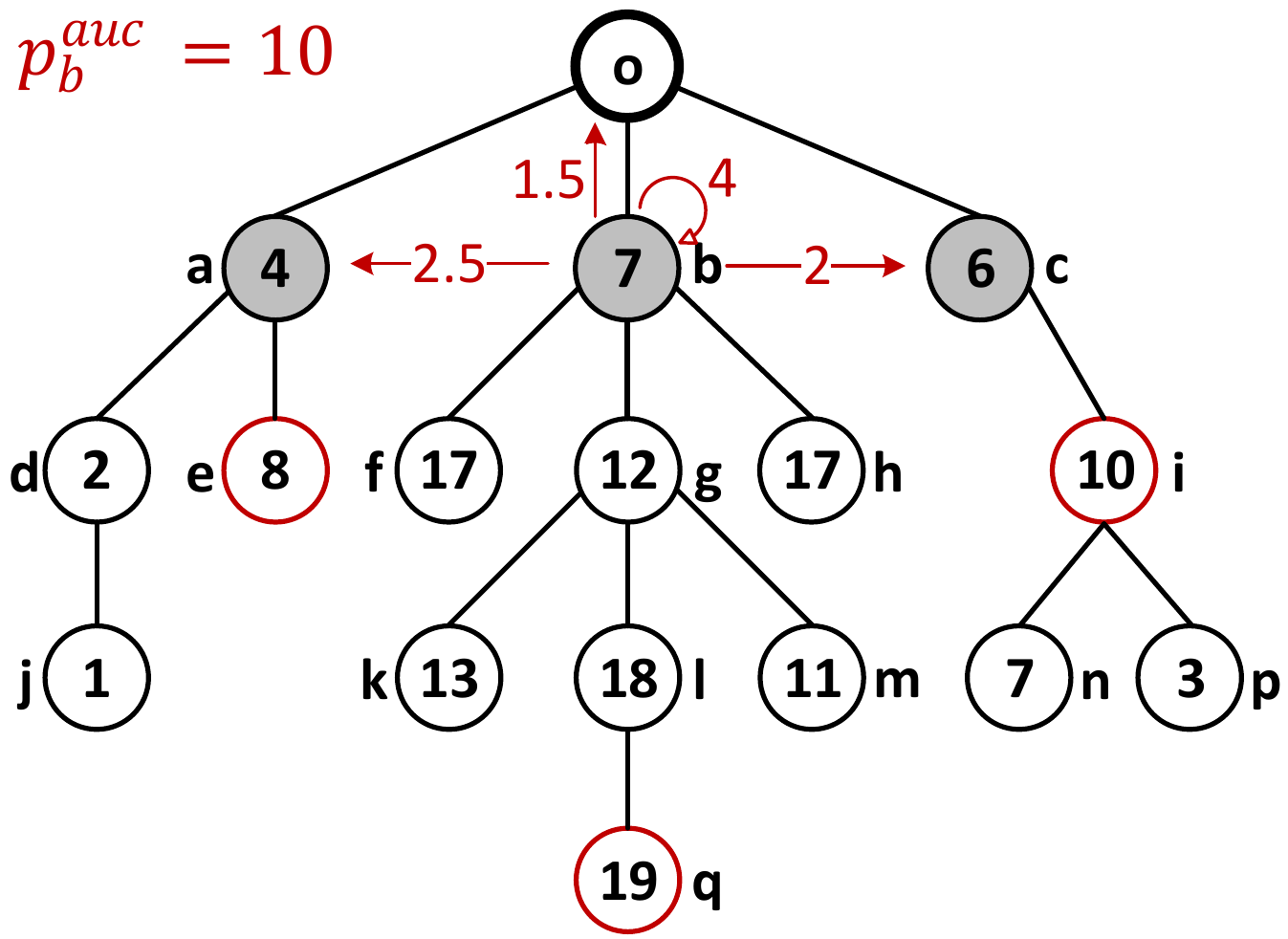}}\\
	\subfigure[]{%
		\label{chart:NRM_3}%
		\includegraphics[width=0.5\linewidth]{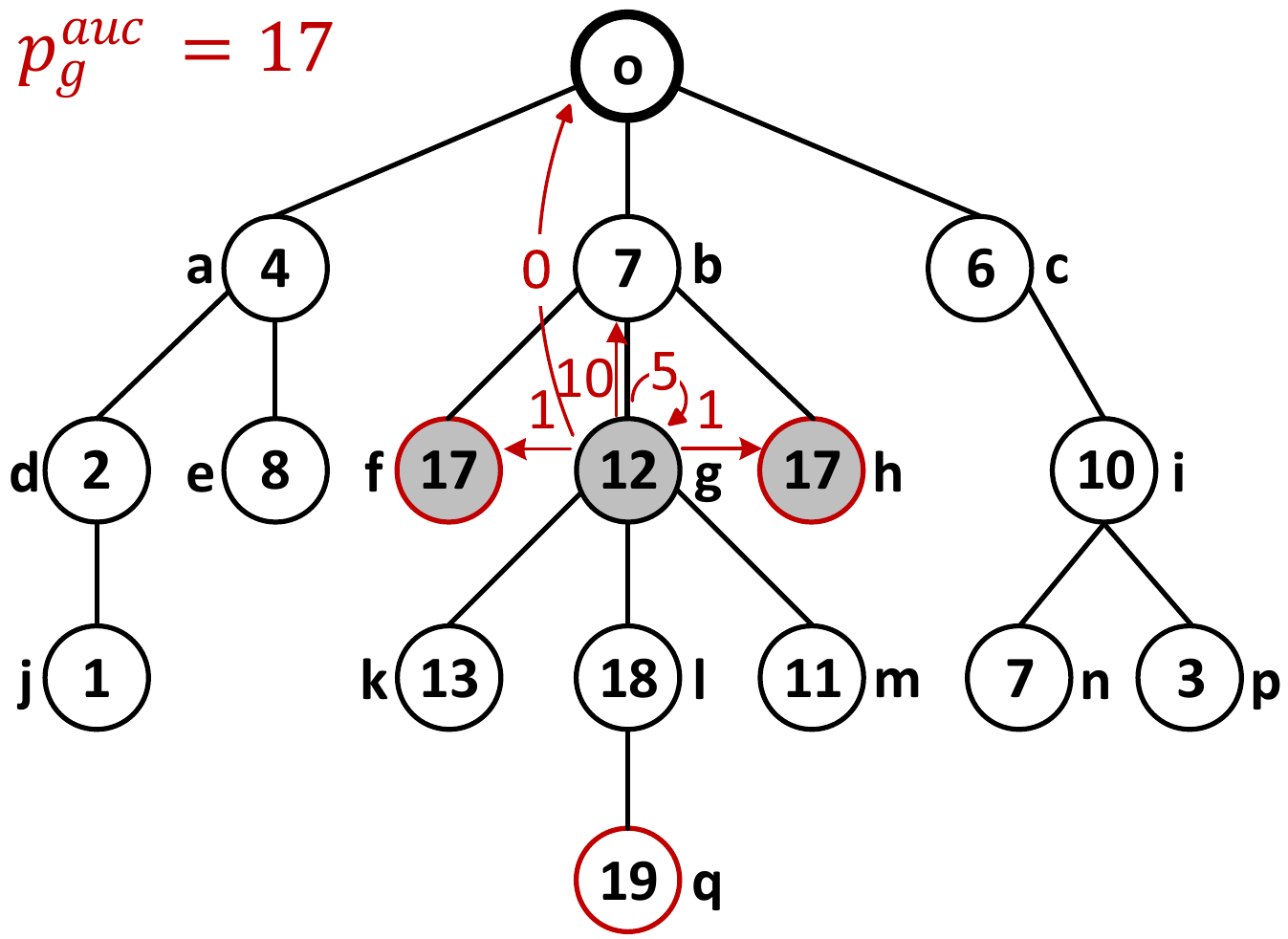}}%
	\subfigure[]{%
		\label{chart:NRM_4}%
		\includegraphics[width=0.5\linewidth]{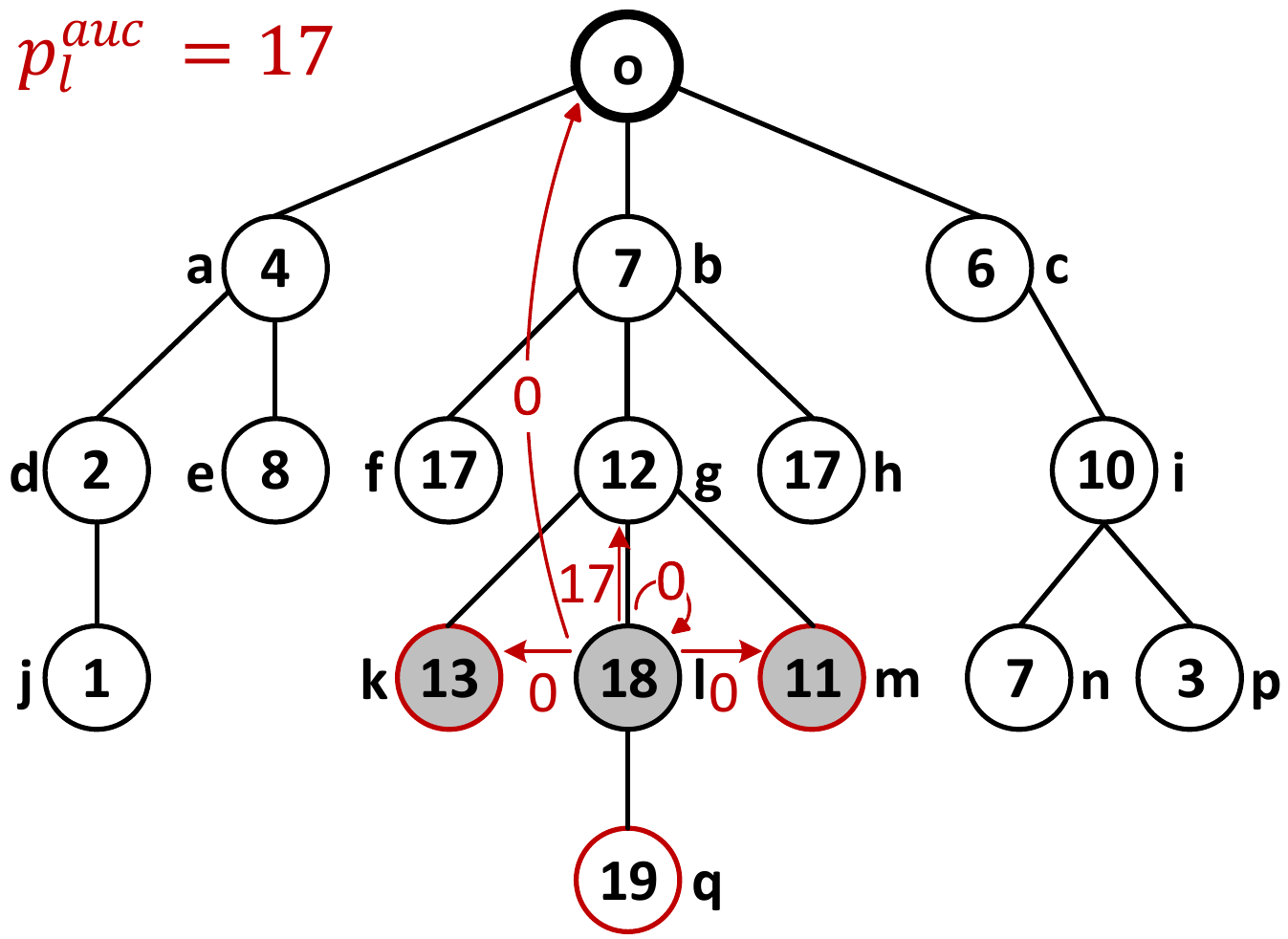}}\\
	\caption{(a) The ancestor sequence of agent $q$ and their siblings; (b)(c)(d) A running process of NRM in Trees.}
	\label{NRM_process}
\end{figure}

We take Figure~\ref{NRM_process} as an example. In Figure~\ref{chart:NRM_1}, the highest bidder is agent $q$ and the ancestor sequence is $A_q=\{b,g,l\}$, which are colored in green. Those purple nodes are siblings of the ancestors, i.e., $B_b=\{a,c\}$, $B_g=\{f,h\}$ and $B_l=\{k,m\}$. Figure~\ref{chart:NRM_2},~\ref{chart:NRM_3}and~\ref{chart:NRM_4} shows a running process of NRM, where each subfigure represents the computational process for each step. For each subfigure, the nodes in grey are the ancestor and its siblings we focus on in this step, and the nodes in red circle are the highest bidder in the subtree of the ancestor or its siblings. The value $P^{auc}$ on the left top is the required payment for the ancestor and the red arrows represent the payment transfer. In the first step in Figure~\ref{chart:NRM_2}, the required payment for agent $b$ is the highest reported valuation without her participation $p_b^{auc}=\hat{v}_i=10$. Since she is the first one in the ancestor sequence, the payment $p_b^{auc}$ can be directly used to redistribute among the siblings and herself. The total number of agents in the subtrees of $a$, $b$ and $c$ is $16$. For agent $a$, without her participation the new required payment is also $10$ and the number of agents in her subtree is $n_a=4$, thus the money redistributed to her is $R_a=\frac{4*10}{16}=2.5$. For agent $b$, the new required payment becomes $8$ if she quits the mechanism and $n_b=8$, thus the money redistributed is $R_b=\frac{8*8}{16}=4$. Similarly, we have $R_c=\frac{4*8}{16}=2$. Then the surplus $p_b^{auc}-R_a-R_b-R_c=10-2.5-4-2=1.5$ will be given to the owner. In the second step in Figure~\ref{chart:NRM_3}, we have $p_g^{auc}=17$. Since $p_b^{auc}=10$, the money first compensates agent $b$'s payment. Then the payment difference $p_g^{auc}-p_b^{auc}=17-10=7$ will be redistributed among $f$, $g$ and $h$. For each $f$, $g$ and $h$, the new required payment difference without their participation is also $7$. Thus, we have $R_f=R_h=\frac{1*7}{7}=1$ and $R_g=\frac{5*7}{7}=5$. In this step the surplus to owner is zero since $p_g^{auc}-p_b^{auc}-R_f-R_g-R_h=0$. In the third step in Figure~\ref{chart:NRM_4}, $p_l^{auc}=17$ is equal to $p_g^{auc}$. Thus the money will all be used to compensate agent $g$'s payment and remain nothing for redistribution. Thus we have $R_k=R_l=R_m=0$. Since $\hat{v}_l=18>p_l^{auc}$, the item will be allocated to agent $l$. Till now, the NRM runs over. The winner is agent $l$ and the surplus is $S=1.5$. Compared to the classical Cavallo mechanism, the owner allocates the item to agent $b$ with social welfare $\hat{v}_b=7$ and only three agents $a$, $b$ and $c$ have positive utilities while in NRM the social welfare is $18$ and $7$ agents have positive utilities. Therefore, our mechanism achieves a more efficient allocation and more agents have positive utilities.

\subsection{Properties of NRM}

In what follows, we show that our NRM satisfies all the desirable properties of IR, IC, ND and ABB in trees. Also, the allocation is more efficient than traditional Cavallo mechanism among the neighbours. 

\begin{theorem}\label{thm:IR}
	The network-based redistribution mechanism in trees is individually rational.
\end{theorem}
\begin{proof}
	According to the algorithm of NRM, for each sibling $i$ of ancestors, they are not required to pay money and they will receive the money redistributed. Thus we have $u_i(\theta_i,\hat{\theta})=\pi_i(\hat{\theta})v_i-p_i(\hat{\theta})=-p_i(\hat{\theta})=R_i(\hat{\theta})\geq0$. For each ancestor $i$ of the winner, although they are required to pay money for either passing or keeping the item, their payment will be compensated by the next ancestor in the sequence. Thus they will be only redistributed the money from the mechanism, i.e., $u_i(\theta_i,\hat{\theta})=\pi_i(\hat{\theta})v_i-p_i(\hat{\theta})=-p_i(\hat{\theta})=R_i(\hat{\theta})\geq0$. For the winner $w$ of the item, her valuation must be greater than or equal to her required payment according to the allocation condition. Together with the redistributed money, her utility is $u_w(\theta_w,\hat{\theta})=\pi_w(\hat{\theta})v_w-p_w(\hat{\theta})=v_w-p_w^{auc}+R_i(\hat{\theta})\geq R_i(\hat{\theta})\geq0$. All the other agents pay nothing. Thus, agents' utilities in NRM are non-negative and the mechanism is individually rational.
\end{proof}

Theorem~\ref{thm:IR} proves that all the agents participating the mechanism will not have negative utilities as long as they report their valuation truthfully, which is the basic requirement for the mechanism. Then we prove that our mechanism is truthful.

\begin{theorem}
	The network-based redistribution mechanism in trees is incentive compatible.
\end{theorem}
\begin{proof}
	As defined in the extended IC, all agents are required not only to report their truthful valuation but also to invite all their neighbours. Here we prove the theorem in two steps. First, fix whatever valuation for each agent, we prove that inviting all the neighbours is the dominant strategy. Next, fix whatever neighbours invited by each agent, we prove that reporting the truthful valuation is the dominant strategy. Thereby, for each agent, both reporting the truthful valuation and inviting all the neighbours is the dominant strategy.
	
	In NRM, all the agents can be divided into four categories: the winner, winner's ancestors, siblings of the ancestors and the others. Only the agents in the first three categories will gain non-zero utilities.
	
	For the winner $w$, her utility is $u_w(\theta_w,\hat{\theta})=\pi_w(\hat{\theta})v_w-p_w(\hat{\theta})=v_w-p_w^{auc}+R_w$. First, assume that her reported valuation is fixed and the neighbours she invited is $\hat{r}_w\subset r_w$. According to the allocation condition, no matter how many neighbours she invites, she will be still the winner since her valuation is at least equal to the required payment. The term $v_w-p_w^{auc}$ remains the same. However, since $R_w$ is a monotone increasing function to the number of the descendants, $R_w$ will decrease if inviting fewer neighbours, which leads to a lower utility. Next, assume that her neighbours invited is fixed and her reported valuation is $\hat{v}_w\not=v_w$. If she is still the winner, her utility remains the same since it is not related to her reported valuation. If she becomes an ancestor of the new winner or the siblings of the ancestor, her utility will only consist the redistributed part $R_w\leq v_w-p_w^{auc}+R_w$, which is lower than the utility of being the winner.
	
	For the winner's ancestor $i$, her utility is $u_i(\theta_i,\hat{\theta})=\pi_i(\hat{\theta})v_i-p_i(\hat{\theta})=R_i$. First, assume that her reported valuation is fixed and the neighbours she invited is $\hat{r}_i\subset r_i$. According to the allocation condition, no matter how many neighbours she invites, she cannot be the winner since her valuation is lower than the required payment, i.e., $v_w<p_w^{auc}$. If she is still the ancestor or becomes a sibling, her utility will decrease after inviting fewer neighbours since the total amount of the money to be redistributed will not increase and the $R_w$ is a monotone increasing function to the number of the descendants. Next, assume that her neighbours invited is fixed and her reported valuation is $\hat{v}_i\not=v_i$. She has no chance to be the sibling. If she becomes the winner, her utility will be $v_i-p_i^{auc}+R_i<R_i$, which is lower than that of reporting truthfully. If she is still the ancestor, her utility will not change no matter what valuation she reports. 
	
	For the sibling of the ancestors $i$, her utility is $u_i(\theta_i,\hat{\theta})=\pi_i(\hat{\theta})v_i-p_i(\hat{\theta})=R_i$. First, assume that her reported valuation is fixed and the neighbours she invited is $\hat{r}_i\subset r_i$. She has no chance to be the winner or the ancestor according to the allocation condition. If she is still the sibling, her utility will decrease after misreporting since the total amount of the money to be redistributed will not increase and the $R_w$ is a monotone increasing function to the number of the descendants. Next, assume that her neighbours invited is fixed and her reported valuation is $\hat{v}_i\not=v_i$. She has no chance to be an ancestor. If she becomes the winner, her required payment will be the highest valuation without her participation, which is higher than her valuation. Thus her utility will decrease because $v_i-p_i^{auc}+R_i<R_i$. If she is still the sibling, her utility remains unchanged.
	
	For any other agent $i$, her utility is $0$. First, assume that her reported valuation is fixed and the neighbours she invited is $\hat{r}_i\subset r_i$. The allocation will not change and she cannot become the winner, the ancestor or the sibling. So she will still gain nothing. Next, assume that her neighbours invited is fixed and her reported valuation is $\hat{v}_i\not=v_i$. The only possible way to gain something through misreporting valuation is to report a higher valuation and become the winner. However, the money she is required to pay must be higher than her valuation and the money redistributed to her must be zero. Thus her utility is negative.
	
	Accordingly, NRM is incentive compatible since all the agents have no incentive to either misreporting their valuation or inviting fewer neighbours.
\end{proof}

Next, we show that the resource owner will never pay some extra money for the allocation in our mechanism and the surplus after redistribution among the agents will go to zero asymptotically when the number of participants goes to infinite. Moreover, our mechanism can achieve a more efficient allocation compared to the Cavallo mechanism.

\begin{theorem}
	The network-based redistribution mechanism in trees runs no deficit.
\end{theorem}
\begin{proof}
	According to the process of NRM, in each step, the ancestor $a_i$ pays the money required $p_{a_i}^{auc}=\hat{v}_{V\setminus V_{a_i}}^{(1)}$ and shares the remaining part after compensation $p_{a_i}^{auc}-p_{a_{i-1}}^{auc}$ among $X=B_{a_i}\cup a_i$. The total money redistributed is 
	\begin{align*}
		\sum_{k\in X}R_k&=\sum_{k\in X}\frac{n_k}{n_X}\cdot S_{-k}\\&\leq\sum_{k\in X}\frac{n_k}{n_X}\cdot (p_{a_i}^{auc}-p_{a_{i-1}}^{auc})\\&=p_{a_i}^{auc}-p_{a_{i-1}}^{auc}
	\end{align*}
	
	Thus, the required payment can cover the compensation and the money redistributed. So NRM runs no deficit.
\end{proof}

\begin{theorem}
	The network-based redistribution mechanism in trees is asymptotically budget-balanced.
\end{theorem}
\begin{proof}
	In each step, the money is redistributed among the ancestor $a_i$ and her siblings $B_{a_i}$. Let the $X=B_{a_i}\cup\{a_i\}=\{x_1,x_2,\cdots,x_m\}$, where $\hat{v}_{V_{x_1}}^{(1)}\geq\hat{v}_{V_{x_2}}^{(1)}\geq\cdots\geq\hat{v}_{V_{x_m}}^{(1)}$.  
	The amount of the money redistributed is 
	\begin{align*}
		&\scalebox{1}{$\sum_{k\in X}R_k$}\\
		=&\scalebox{1}{$\sum_{k\in X}\frac{n_k}{n_X}\cdot S_{-k}$}\\
		=&\scalebox{1}{$\frac{n_{x_1}}{n_X}\cdot \max(0,\hat{v}_{V_{x_3}}^{(1)}-p_{a_{i-1}}^{auc})+\frac{n_{x_2}}{n_X}\cdot\max(0, \hat{v}_{V_{x_3}}^{(1)}-p_{a_{i-1}}^{auc})$}\\
		&\scalebox{1}{$+\frac{n_{x_3}}{n_X}\cdot \max(0,\hat{v}_{V_{x_2}}^{(1)}-p_{a_{i-1}}^{auc})+\cdots+\frac{n_{x_m}}{n_X}\cdot \max(0,\hat{v}_{V_{x_2}}^{(1)}-p_{a_{i-1}}^{auc})$}\\
		=&\scalebox{1}{$\lambda\cdot \max(0,\hat{v}_{V_{x_3}}^{(1)}-p_{a_{i-1}}^{auc})+(1-\lambda)\cdot \max(0,\hat{v}_{V_{x_2}}^{(1)}-p_{a_{i-1}}^{auc})$}
	\end{align*}
	where $\lambda=(n_{x_1}+n_{x_2})/n_X$. 
	
	As the number of participating agents $n$ goes to $\infty$, the surplus for this step is 
	\begin{align*}
		&\scalebox{1}{$\lim\limits_{n \to +\infty} (p_{a_i}^{auc}-p_{a_{i-1}}^{auc}-\sum_{k\in X}R_k)$}\\
		=&\scalebox{1}{$\lim\limits_{n \to +\infty} (\max(\hat{v}_{V_{x_2}}^{(1)},p_{a_{i-1}}^{auc})-p_{a_{i-1}}^{auc}$}\\
		&\scalebox{1}{$-\lambda\cdot \max(0,\hat{v}_{V_{x_3}}^{(1)}-p_{a_{i-1}}^{auc})-(1-\lambda)\cdot \max(0,\hat{v}_{V_{x_2}}^{(1)}-p_{a_{i-1}}^{auc}))$}\\
		=&\scalebox{1}{$\lim\limits_{n \to +\infty} (\max(0,\hat{v}_{V_{x_2}}^{(1)}-p_{a_{i-1}}^{auc})$}\\
		&\scalebox{1}{$-\lambda\cdot \max(0,\hat{v}_{V_{x_3}}^{(1)}-p_{a_{i-1}}^{auc})-(1-\lambda)\cdot \max(0,\hat{v}_{V_{x_2}}^{(1)}-p_{a_{i-1}}^{auc}))$}\\
		=&\scalebox{1}{$\lim\limits_{n \to +\infty} (\lambda\cdot\max(0,\hat{v}_{V_{x_2}}^{(1)}-p_{a_{i-1}}^{auc})-\lambda\cdot \max(0,\hat{v}_{V_{x_3}}^{(1)}-p_{a_{i-1}}^{auc}))$}\\
		=&\scalebox{1}{$\lim\limits_{n \to +\infty} \frac{n_{x_1}+n_{x_2}}{n_X}\cdot(\max(0,\hat{v}_{V_{x_2}}^{(1)}-p_{a_{i-1}}^{auc})- \max(0,\hat{v}_{V_{x_3}}^{(1)}-p_{a_{i-1}}^{auc}))$}\\
		=&\scalebox{1}{$0$}
	\end{align*}
	Thus in each step the surplus is asymptotically zero, so the NRM is asymptotically budget-balanced.
	
\end{proof}

\begin{theorem}\label{efficiency}
	The network-based redistribution mechanism in trees is at least as efficient as Cavallo mechanism among neighbours.
\end{theorem}
\begin{proof}
	According to the allocation condition, the winner is the agent whose reported valuation satisfies $\hat{v}_w\geq p_w^{auc}=\hat{v}_{V\setminus V_w}^{(1)}\geq\hat{v}_{r_o}^{(1)}$. Thus, NRM is at least as efficient as Cavallo mechanism among the owner's neighbours.
\end{proof}

\subsection{Efficiency Impossibility}
In what follows, we discuss the efficiency of redistribution mechanisms on networks. It seems that Theorem~\ref{efficiency} is not that strong since there are no further guarantees for the efficiency except for the improvement compared to the Cavallo mechanism among the neighbours. However, even if we ignore redistribution and only consider non-deficit (and IC, IR), efficiency approximation has not been found yet for diffusion settings~\cite{DBLP:conf/aaai/LiHZZ17,DBLP:conf/ijcai/LiHZY19}.

Let us now exhibit the negative result regarding the design of mechanisms for diffusion settings. The result shows that it is impossible to achieve any efficiency guarantee given that all the properties are satisfied.

\begin{proposition}\label{impo}
	There exists no mechanism in the diffusion setting which can guarantee any efficiency without sacrificing non-deficit, IR and IC.
\end{proposition}

\begin{proof}
	Consider a simple line graph in Figure~\ref{fig:worst}, where $o$ is the resource owner, agent $a$ has two neighbors $o$ and $b$, and $a$ and $b$'s valuations are $v_a$ and $v_b$ and $v_a<v_b$. If $a$ does not invite $b$, $a$ wins the resource and pays zero, so $a$ will only invite $b$ if $a$'s reward is at least $v_a$. If $a$ invites $b$, to achieve a more efficient allocation, $b$ wins and her payment should be not more than $v_b$; otherwise, it may violate IR. To achieve non-deficit, $b$ should pay at least what $a$ receives. Eventually, their payments depend on their valuations, which violates IC. Thus, no matter how large $v_b$ is, she cannot be allocated the item.
	
	That is, whatever the mechanism, it cannot guarantee any efficiency for all the network structures without sacrificing non-deficit, IR and IC.
\end{proof}

\begin{figure}[htbp]
	\centering
	\includegraphics[width=0.5\linewidth]{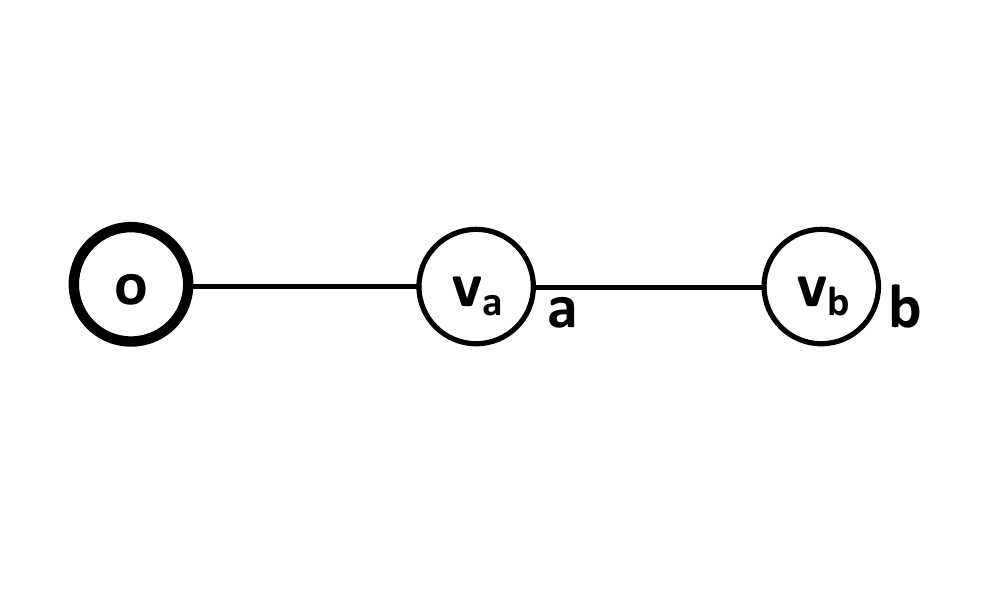}
	\caption{The worst case for efficiency.}
	\label{fig:worst}
\end{figure}

The simple setting presented in proposition~\ref{impo} is somewhat equivalent to the bilateral trading setting of one seller and one buyer studied by~\citeauthor{myerson1983efficient}~\shortcite{myerson1983efficient} where $a$ behaves like the seller and $b$ is the buyer since $a$ can easily get the item for free by not inviting anyone. Then their well-known impossibility theorem holds here (i.e. we cannot have efficiency, IC, IR and non-deficit at the same time). Even in bilateral trading settings, we haven't seen good non-deficit examples to approximate efficiency. The well-known example is McAfee's trade reduction for multiple buyers and multiple sellers, where efficiency is sacrificed to remove deficit, but the efficiency loss is diminished when the number of traders increases~\cite{mcafee1992dominant}. However, it still does not guarantee a lower bound of efficiency in general. In the worst case when there is only one seller and one buyer, it has no efficiency guarantee.

Therefore, the efficiency of a mechanism depends on the network. To the best of our knowledge, we are the very first to study the redistribution problem in the network setting. The benchmark used in our setting is the Cavallo mechanism, which achieves efficiency among the resource owner's neighbours. Since we have shown that generically, no non-deficit, IR and IC mechanism can guarantee any efficiency in the diffusion setting in Proposition~\ref{impo}, NRM breaks new ground for this problem as the efficiency is improved compared to Cavallo mechanism. 

\section{Redistribution Mechanism in Graphs}
\label{section:Mech_Graph}
In the previous section, we only studied the mechanism in tree structures. In real life, most social networks are common graphs. Hence in this section, we extend our NRM to more general cases without sacrificing all desirable properties.

Different from the tree cases, we extend the definitions and basic notations for the graph setting.

\begin{definition}
	Given a reported type profile $\hat{\theta}\in \Theta$ of all the agents and a generated common graph $G=(V,E)$, for each agent $i,j\in V$ if all the paths from $o$ to $j$ have to pass $i$, we say $i$ is $j$'s \textbf{ancestor} and $j$ is $i$'s \textbf{descendant}.
\end{definition}

\begin{itemize}[leftmargin= 20 pt]
	\item Let $A_i=(a_1,\cdots,a_k)$ be the \textbf{ancestor sequence} of agent $i$.
	\item Let $G_i=(V_i,E_i)$ be the \textbf{subgraph} of agent $i$.
	\item Let $B_{a_j}=r_{a_{j-1}}^c\setminus a_j$ be the \textbf{sibling set} of agent $a_j\in A_i$.
\end{itemize}

That is, an ancestor $a_j\in A_i$ for agent $i\in V$ is a cut-point from the seller to $i$. The subgraph of agent $i$ are those who cannot receive the information without $i$'s invitation. The siblings of an ancestor $a_j$ are the child neighbours of ancestor $a_{j-1}$ except $a_j$ herself.

Then the NRM can be simply extended in graphs by updating the definitions of the notations above in Algorithm~\ref{NRM-T}.

Intuitively, the network-based redistribution mechanism in graphs is a generalization of that in trees. The sibling set who share the money with an ancestor are the child neighbours of the last ancestor. Seemingly, it is quite different from the ancestor's brother neighbours with the same parent in tree cases. Actually, in tree structures, the ancestor herself is also one of the child neighbours of the last ancestor, which can be viewed as a special case of the common graphs.

\begin{figure}[htbp]
	\centering
	\subfigure[]{%
		\label{chart:NRM_5}%
		\includegraphics[width=0.5\linewidth]{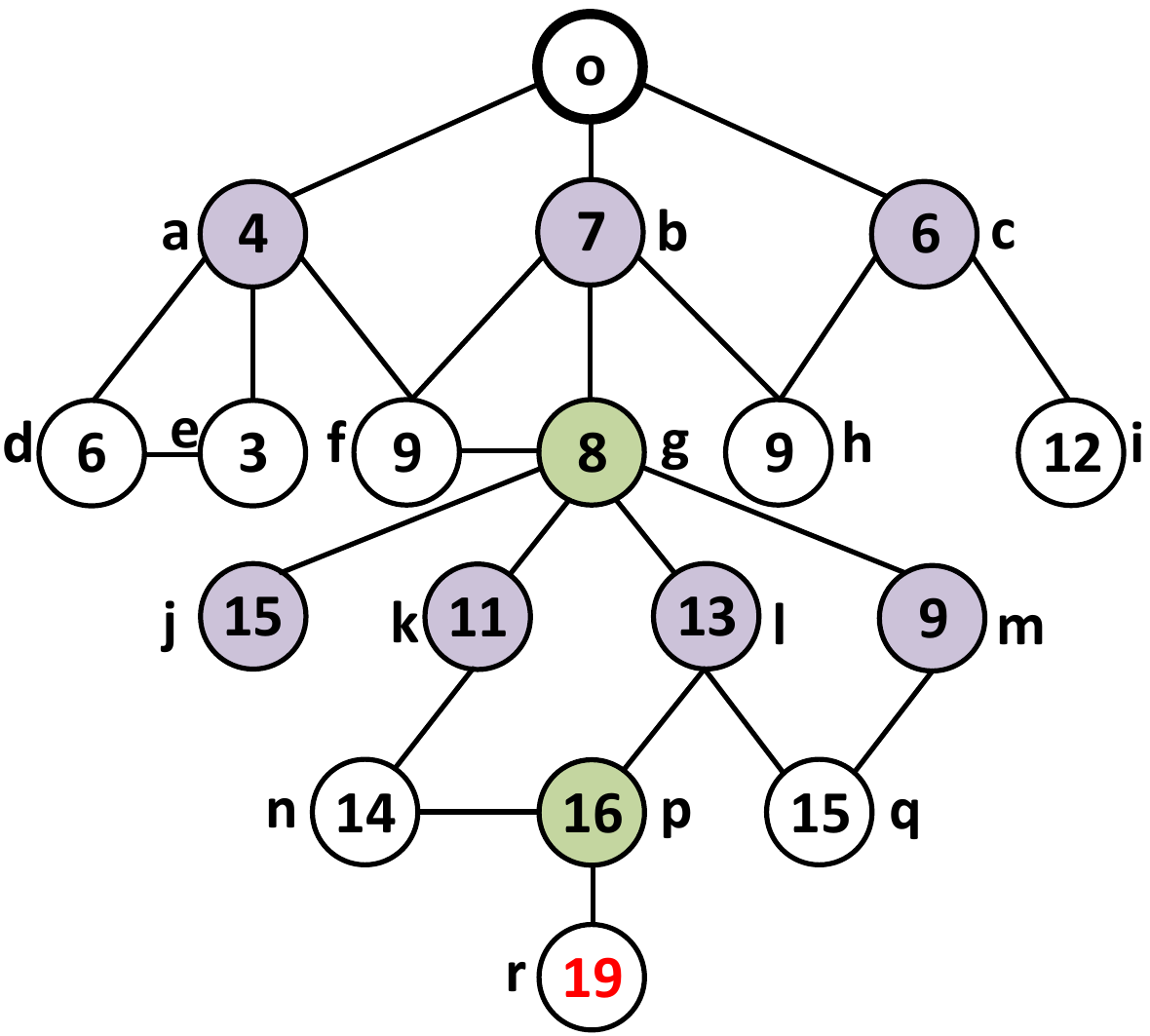}}
	\subfigure[]{%
		\label{chart:NRM_6}%
		\includegraphics[width=0.5\linewidth]{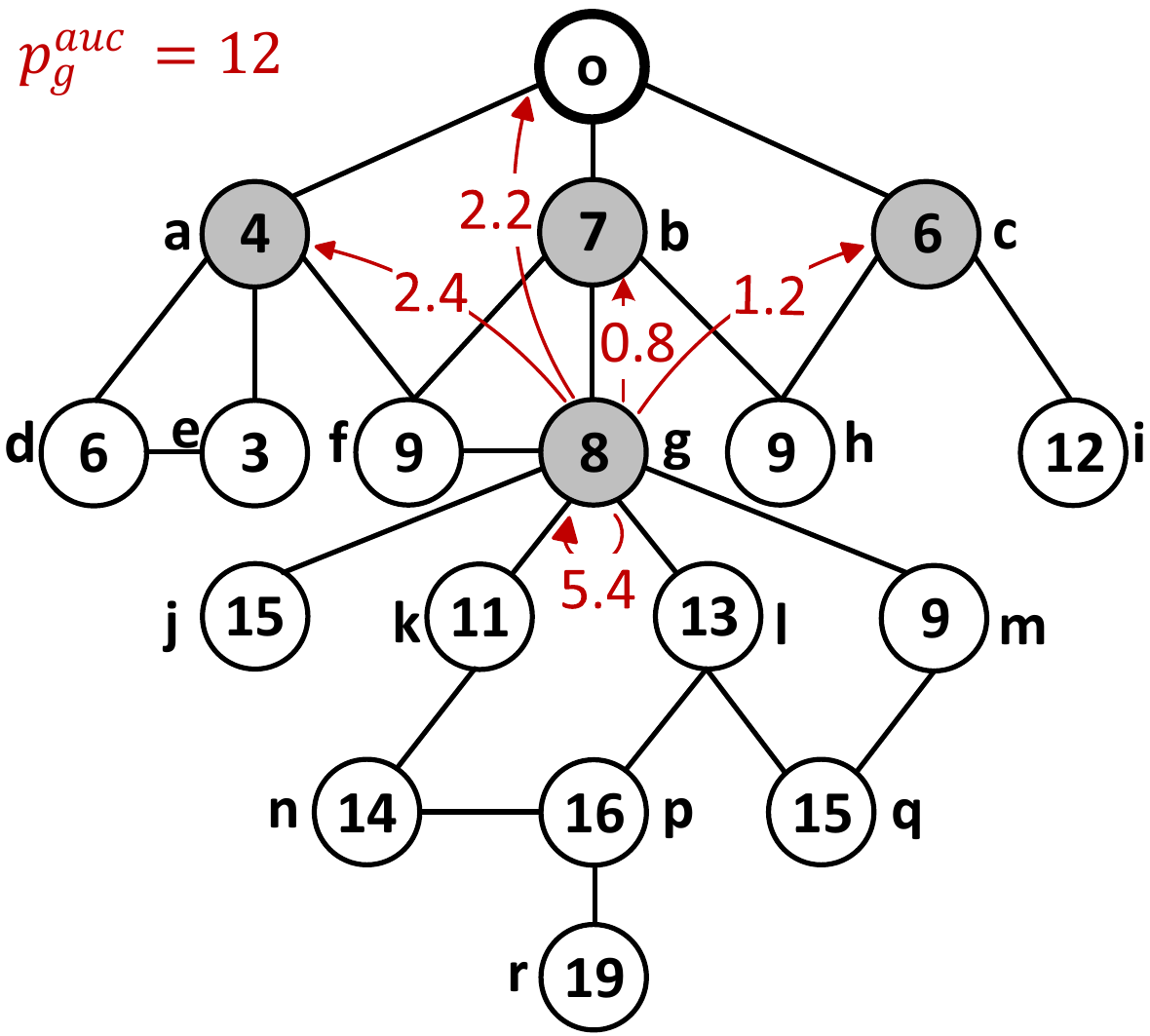}}%
	\subfigure[]{%
		\label{chart:NRM_7}%
		\includegraphics[width=0.5\linewidth]{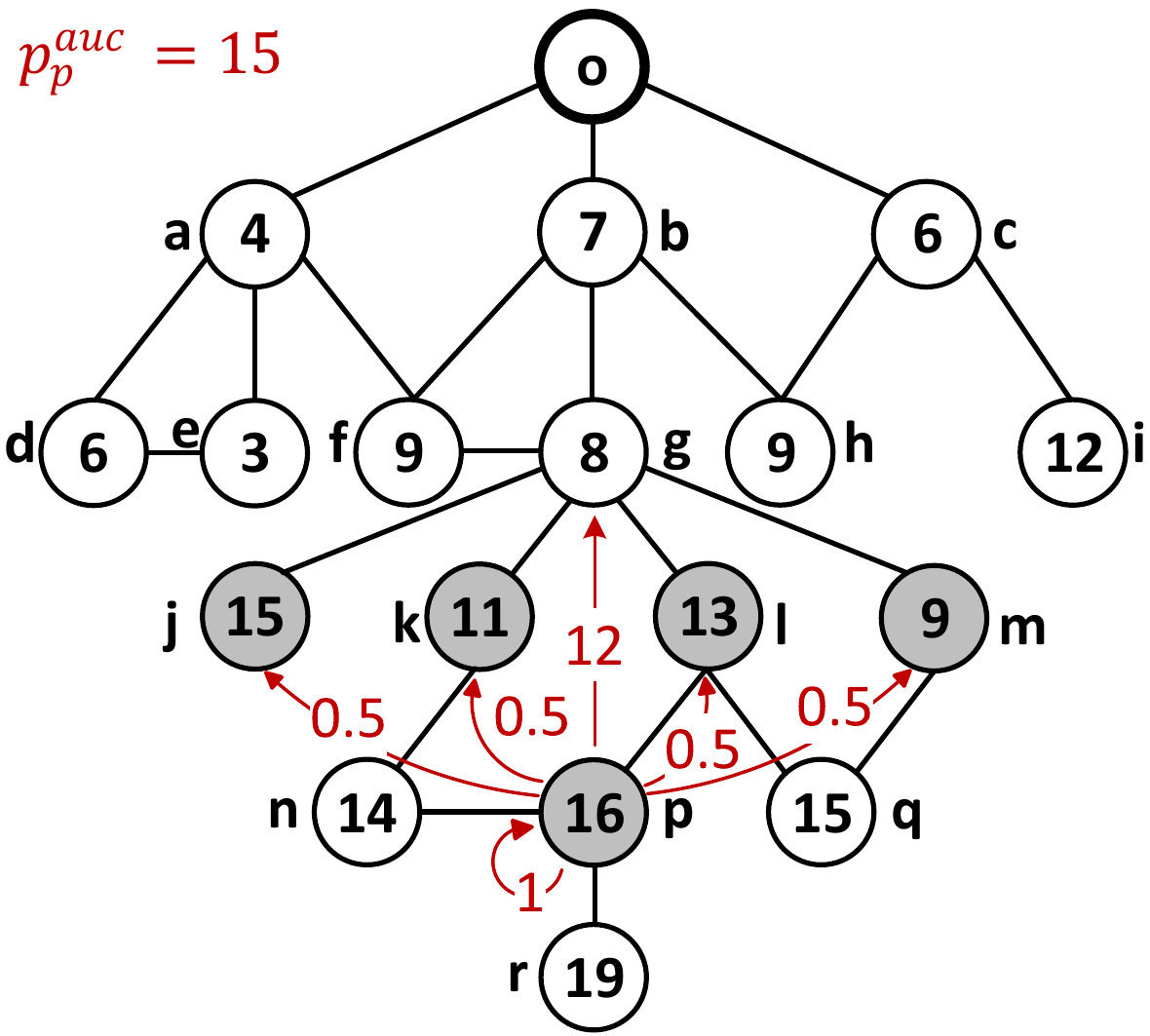}}
	\caption{(a) The ancestor sequence of agent $r$ and their siblings; (b)(c) A running process of NRM in graphs.}
	\label{NRM_process_2}
\end{figure}

Also an example is given to illustrate the mechanism. In Figure~\ref{chart:NRM_5}, the highest bidder is agent $r$. According to the definition, the green nodes are the ancestors of agent $r$ and purple nodes are their siblings, i.e., $A_r=\{g,p\}$, $B_g=\{a,b,c\}$ and $B_p=\{j,k,l,m\}$. In the running example in Figure~\ref{chart:NRM_6}~\ref{chart:NRM_7}, first for agent $g$, her required payment $p_g^{auc}=\hat{v}_i=12$ will all be used to redistribute among $a$, $b$, $c$ and $g$. The total number of agents in the subgraph of all $a$, $b$, $c$ and $g$ is $3+1+2+9=15$. If $a$ quits the mechanism, $d$ and $e$ will be not able to receive the information and $b$ will be the new ancestor with payment $12$. Thus we have $R_a=\frac{3*12}{15}=2.4$. If $b$ quits the mechanism, only herself will be out of the network and $a$ will be the new ancestor with payment $12$. Thus we have $R_b=\frac{1*12}{15}=0.8$. Similarly, we have $R_c=\frac{2*9}{15}=1.2$ and $R_g=\frac{9*9}{15}=5.4$. The surplus in this step is $12-2.4-0.8-1.2-5.4=2.2$, which will be given to the owner. In the same way, the required payment for agent $p$ is $p_p^{auc}=15$ and the remaining money after compensation is $p_p^{auc}-p_g^{auc}=3$. Then the money redistributed to each agent is $R_j=R_k=R_l=R_m=\frac{1*3}{6}=0.5$ and $R_p=\frac{2*3}{6}=1$. The surplus is $0$ and $p$ will keep the item since $\hat{v}_p=16\geq 15=p_p^{auc}$. Till now, the mechanism runs over. The winner is agent $p$, the social welfare is $\hat{v}_p=16$ and the total surplus is $2.2$. Compared to the classical Cavallo mechanism, agent $b$ is allocated the item with social welfare $\hat{v}_b=7$ and only three agents $a$, $b$ and $c$ have positive utilities while in NRM the social welfare is $16$ and $9$ agents have positive utilities. 

Since the network-based redistribution mechanism in graphs is a generalization of that in trees, we can easily obtain the following corollary.

\begin{corollary}
	The network-based redistribution mechanism in graphs runs no deficit and is IR, IC, ABB and at least as efficient as Cavallo mechanism among neighbours.
\end{corollary}

\section{Conclusion}
\label{section:conclusion}
In this paper, we considered the redistribution mechanism design problem in social networks, where the owner wants to allocate one item and hopes the wealth maintained among the agents as much as possible. The objective is to incentivize agents participated to invite all their neighbours to the mechanism so that the owner can make the allocation more efficient. The classical Cavallo mechanism performs well in the traditional static setting; however, it may lead to a deficit and disincentivize the agents to diffuse the information if it is directly extended in social networks. To overcome the challenge, we propose a novel network-based redistribution mechanism which incentivizes agents to invite all their neighbours. The key of our mechanism is that the reward redistributed to each agent on the network is carefully designed by a monotone increasing function to the number of participants she invites. The mechanism works not only for the tree structures but also for the common graphs. Moreover, the mechanism satisfies all the desirable properties of individual rationality, incentive compatibility, asymptotically budget-balance and non-deficit. More importantly, the allocation is also more efficient. 

Our work has many interesting aspects for further investigation. Although we have shown the efficiency impossibility result in the paper, it is still an interesting future work to study the efficiency approximation for some certain network structures like trees given some prior knowledge about agents' valuations. We only consider the single-item situation in this paper, so it may be a challenge to extend the mechanism for multiple items, which cannot be achieved by simply running the mechanism for several times. We assume that there is no cost for an agent to spread the information to their neighbours in this paper. Another interesting future work may be designing a mechanism for the setting with cost. 

%
%


\bibliographystyle{ACM-Reference-Format}  
\bibliography{sample-bibliography}  

\end{document}